\documentclass[11pt]{article}
\usepackage[margin=1in]{geometry}
\usepackage[usenames,dvipsnames,svgnames,table]{xcolor} 

\usepackage{longtable}
\usepackage{hhline}
\usepackage{float}

\usepackage{setspace}
\usepackage{mathrsfs}
\usepackage{xspace}
\usepackage{enumitem} 
\usepackage{rotfloat} 
\usepackage{graphicx}
\usepackage{amssymb, amsmath, amsthm}
\usepackage{verbatim}
\usepackage{natbib}
\usepackage{authblk}
\usepackage{multirow}
\usepackage{subcaption} 
\usepackage{array}
\usepackage{hyperref}
\newcolumntype{L}{>{\centering\arraybackslash}m{0.1\linewidth}}

\newtheorem{theorem}{Theorem}
\newtheorem{lemma}[theorem]{Lemma}
\newtheorem{proposition}[theorem]{Proposition}
\newtheorem{corollary}[theorem]{Corollary}

\newtheorem{example}{Example}[section]
\newtheorem{definition}{Definition}
\newtheorem{assumption}{Assumption}

\newcommand{\calH}{\mathcal{H}}

\newcommand{\bbS}{\mathbb{S}}

\newcommand{\R}{\mathbb{R}}

\newcommand{\E}{\mathbb{E}}
\newcommand{\Prob}{\mathbb{P}}

\newcommand{\inner}[2]{\left\langle #1, #2 \right\rangle}

\newcommand{\MMD}{\operatorname{MMD}}
\newcommand{\ESS}{\operatorname{ESS}}
\newcommand{\tr}{\operatorname{tr}}
\newcommand{\Cov}{\operatorname{Cov}}
\newcommand{\Var}{\operatorname{Var}}

\newcommand{\iid}{\mathrm{iid}}

\newcommand{\Sphere}{\mathbb{S}}
\newcommand{\Hk}{\mathcal{H}_k}

\newcommand{\toP}{\xrightarrow{\Prob}}

\begin{document}
	
\title{Intrinsic effective sample size for manifold-valued Markov chain Monte Carlo via kernel discrepancy}
\author[1,2]{Kisung You}
\affil[1]{Department of Mathematics, Baruch College}
\affil[2]{Department of Mathematics, The Graduate Center, City University of New York}
\date{}

\maketitle
\begin{abstract}
Effective sample size is a standard summary of Markov chain Monte Carlo output, but it is usually attached to scalar or Euclidean summaries chosen by the analyst. For manifold-valued samples this choice is not canonical: coordinate-wise effective sample sizes can change under rotations, chart changes, or alternative embeddings of the same underlying path. We propose an intrinsic effective sample size based on kernel discrepancy. The proposed quantity is the number of independent draws that would yield the same expected squared kernel discrepancy between the empirical distribution and the target distribution. This gives an exact finite-sample risk interpretation, an asymptotic integrated-autocorrelation representation, and a coordinate-free diagnostic whenever the kernel respects the geometry of the state space. We establish invariance under transported kernels, operator and principal-direction interpretations, and consistency of a lag-window estimator under boundedness and absolute-regularity conditions. We also discuss valid kernel constructions on manifolds, emphasizing that geodesic Gaussian kernels are not generally positive definite on curved spaces. Sphere experiments illustrate rotation invariance and calibration of the proposed diagnostic against empirical distributional error.
\end{abstract}

\section{Introduction}
Effective sample size (ESS) is among the most common summaries reported for Markov chain Monte Carlo (MCMC) output. Formally, ESS is defined through Markov chain central limit theory for scalar or vector-valued functionals and quantifies the loss of efficiency relative to independent sampling. In practice this principle is often instantiated by selecting scalar summaries $f(X_1),\ldots,f(X_n)$ and computing ESS from the associated time series. In Euclidean problems a standard choice is to take $f$ to be a coordinate projection, or to aggregate coordinate-wise ESS values. This reliance on coordinates is often acceptable for a fixed Euclidean representation. Moreover, multivariate ESS, as developed by \citet{vats_2019_MultivariateOutputAnalysis}, removes dependence on nonsingular linear transformations of a chosen vector output. For manifold-valued output, however, the choice of Euclidean coordinates, charts, or embeddings is itself noncanonical. The same stored path can exhibit sharply different coordinate-wise ESS values after a rotation, a change of chart, or a different embedding, even though the underlying chain on the manifold has not changed.

This paper proposes a coordinate-free alternative. The guiding principle is that ESS should be defined relative to an estimation task and a loss. On a manifold, both should respect the geometry. We focus on the most basic task of approximating the target probability law itself. Given a positive-definite kernel $k$ on the state space, the empirical law $\hat\Pi_n=n^{-1}\sum_{t=1}^n\delta_{X_t}$ can be compared with the target law $\Pi$ by the squared maximum mean discrepancy,
\[
\MMD_k^2(\hat\Pi_n,\Pi)
=\left\|\int k(x,\cdot)\,d\hat\Pi_n(x)-\int k(x,\cdot)\,d\Pi(x)\right\|_{\Hk}^2,
\]
where $\Hk$ is the reproducing-kernel Hilbert space associated with $k$  \citep{gretton_2012_KernelTwoSampleTest}. We define the kernel ESS to be the independent sample size that gives the same expected value of this loss. This yields a scalar summary, but one attached to a distributional and coordinate-free task rather than to an arbitrary coordinate projection.

The central identity is simple but useful. If
\[
\phi(x)=k(x,\cdot)-\int k(u,\cdot)\,\Pi(du)
\]
is the centered feature map and
$
\gamma_\ell=\E[\inner{\phi(X_0)}{\phi(X_\ell)}_{\Hk}],
$
then stationarity gives the exact finite-sample formula
\[
\E[\MMD_k^2(\hat\Pi_n,\Pi)]
=
\frac{1}{n^2}\left\{n\gamma_0+2\sum_{\ell=1}^{n-1}(n-\ell)\gamma_\ell\right\}.
\]
Thus the independent-sample risk $\gamma_0/n$ is inflated or deflated by the same kind of lag-covariance structure that underlies classical integrated autocorrelation time. Under summability of the $\gamma_\ell$'s, the asymptotic risk constant is
\[
\sigma_k^2=\gamma_0+2\sum_{\ell=1}^\infty\gamma_\ell,
\]
and the asymptotic kernel ESS is $n\gamma_0/\sigma_k^2$, with the usual convention that negative dependence can yield an ESS exceeding $n$.

The proposal has three advantages. First, once the kernel is chosen intrinsically, the resulting ESS is invariant under coordinate transformations and isometries. Second, it is distributional: it measures effective independent sample size for approximating the target law under a kernel loss. When the kernel is characteristic, this loss is a genuine metric on probability laws \citep{sriperumbudur_2011_UniversalityCharacteristicKernels}. Third, the definition has a finite-$n$ risk interpretation, which leads naturally to plug-in estimators and to a mean-square precision criterion for MCMC output.

The construction is not kernel-free, which is intentional. ESS is task-dependent even in Euclidean output analysis. Here, the kernel determines which distributional discrepancies matter. The role of the paper is to define and analyze the effective independent sample size associated with that specified intrinsic discrepancy. The choice of kernel is therefore an inferential choice, analogous to the choice of summary in scalar ESS or the choice of norm and covariance functional in multivariate ESS. Connections between RKHS discrepancies and distance-based statistics further support this distributional interpretation \citep{sejdinovic_2013_EquivalenceDistancebasedRKHSbased}.

For the potentially curved state space, kernel validity is a serious issue. It is not enough to take a geodesic distance $d_g$ and write
$
k(x,y)=\exp\{-d_g(x,y)^2/h^2\}.
$
On a curved manifold this function is generally not positive definite. \citet{feragen_2015_GeodesicExponentialKernels} showed that positive definiteness of the geodesic Gaussian for all bandwidths forces a flatness condition. Hence geodesic Gaussians are not generally safe defaults. We therefore discuss two safe construction principles: the Schoenberg--Gneiting theory of isotropic positive-definite functions on spheres \citep{schoenberg_1942_PositiveDefiniteFunctions, gneiting_2013_StrictlyNonstrictlyPositive}, and pullback kernels induced by embeddings or diffeomorphisms into Euclidean spaces. The latter covers projection kernels on Grassmann manifolds, log-Euclidean kernels on SPD manifolds, and Cholesky-based kernels on correlation manifolds.

The remainder of the paper is organized as follows.  Section~\ref{sec:definition} defines kernel ESS as a distributional risk ratio and records the independent-sampling benchmark.  Section~\ref{sec:population} develops the population theory, including the exact finite-sample MMD risk identity, the asymptotic integrated-autocorrelation representation, invariance under transported kernels, and operator interpretations.  Section~\ref{sec:estimation} introduces a Gram-matrix lag-window estimator, gives a consistency result, and describes a kernel-MMD precision rule.  Section~\ref{sec:kernels} discusses valid kernel constructions on manifolds, with particular attention to spheres and pullback kernels.  Section~\ref{sec:experiments} presents numerical experiments on the sphere, and Section~\ref{sec:discussion} concludes with practical considerations and extensions.  All proofs are collected in the Appendix.

\section{Kernel ESS as a distributional risk ratio}
\label{sec:definition}

We begin with the definition in a general state-space setting. The purpose of this abstraction is not to obscure the geometry, but to separate the probabilistic identity from the particular manifold on which it will be used. The manifold enters through the kernel: if the kernel is invariant under the relevant geometry, then the resulting ESS inherits that invariance.

Let $(M,\mathscr B)$ be a measurable space. In applications, $M$ will be a Riemannian manifold with its Borel $\sigma$-field. Let $\{X_t\}_{t\in\mathbb Z}$ be a strictly stationary process with marginal law $\Pi$. Working with a two-sided stationary version simplifies notation. All definitions for an observed path $X_1,\ldots,X_n$ are one-sided. For MCMC applications, stationarity is the standard idealization after burn-in.

\begin{assumption}[Kernel regularity]
\label{ass:kernel}
The kernel $k:M\times M\to\R$ is symmetric, positive definite, and bounded on the diagonal:
\[
K_0:=\sup_{x\in M} k(x,x)<\infty.
\]
The feature map $x\mapsto k(x,\cdot)$ from $M$ into the reproducing-kernel Hilbert space $\Hk$ is Bochner measurable.
\end{assumption}

By positive definiteness, $|k(x,y)|\le K_0$ for all $x,y$. Assumption~\ref{ass:kernel} therefore ensures Bochner integrability of the feature map. In the main manifold examples, measurability follows automatically. If $M$ is second-countable and $k$ is continuous, then
\[
\|k(x,\cdot)-k(y,\cdot)\|_{\Hk}^2=k(x,x)+k(y,y)-2k(x,y),
\]
so the feature map is continuous. Since its image is separable due to separability of $M$, the map is Bochner measurable.

Define the kernel mean embedding of $\Pi$ by
\[
\mu_\Pi=\int_M k(x,\cdot)\,\Pi(dx)\in\Hk
\]
and the centered feature map by
\[
\phi(x)=k(x,\cdot)-\mu_\Pi.
\]
Then $\int\phi(x)\,\Pi(dx)=0$ in $\Hk$, and $\|\phi(x)\|_{\Hk}\le 2K_0^{1/2}$. For the empirical law $\hat\Pi_n=n^{-1}\sum_{t=1}^n\delta_{X_t}$,
\begin{equation}
\label{eq:mmd-centred}
\MMD_k^2(\hat\Pi_n,\Pi)
=
\left\|\frac1n\sum_{t=1}^n\phi(X_t)\right\|_{\Hk}^2.
\end{equation}
This is the squared RKHS distance between the empirical measure and the target. It includes the diagonal terms of the empirical measure and should not be confused with the unbiased two-sample $U$-statistic sometimes used to estimate MMD in testing problems.

For $\ell\in\mathbb Z$, define the scalar feature-autocovariance
\begin{equation}
\label{eq:gamma}
\gamma_\ell=\E\inner{\phi(X_0)}{\phi(X_\ell)}_{\Hk}.
\end{equation}
Stationarity and symmetry of the Hilbert inner product imply $\gamma_{-\ell}=\gamma_\ell$. In particular,
\[
\gamma_0=\E\|\phi(X_0)\|_{\Hk}^2.
\]
If $Y_1,\ldots,Y_m$ are independent draws from $\Pi$, then
\begin{equation}
\label{eq:iid-risk}
\E\bigg[\MMD_k^2 \Big(m^{-1}\sum_{i=1}^m\delta_{Y_i},\Pi\Big)\bigg]=\frac{\gamma_0}{m}.
\end{equation}
This identity is the independent benchmark.

\begin{definition}[Exact finite-sample kernel ESS]
\label{def:kernel-ess}
Assume $\gamma_0>0$ and let
\[
R_n(k)=\E[\MMD_k^2(\hat\Pi_n,\Pi)].
\]
The exact finite-sample kernel effective sample size is
\begin{equation}
\label{eq:ess-def}
\ESS_k^{(n)}=
\begin{cases}
\gamma_0/R_n(k),& R_n(k)>0,\\[0.25em]
+\infty,& R_n(k)=0.
\end{cases}
\end{equation}
When $R_n(k)>0$, $\ESS_k^{(n)}$ is the unique positive number $m$ for which the independent risk $\gamma_0/m$ equals the Markov-chain risk $R_n(k)$.
\end{definition}

The condition $\gamma_0>0$ only rules out the degenerate situation in which the centered feature map is $\Pi$-almost surely zero. The ESS is a nonnegative real number, not necessarily an integer. It may exceed $n$ under beneficial negative dependence, just as scalar ESS can exceed $n$.

\section{Population theory}
\label{sec:population}

The definition above is meaningful because the MMD risk has an exact autocovariance expansion. This section develops that expansion and its consequences. The first theorem is finite-sample and requires only stationarity. The second gives the asymptotic integrated-autocorrelation-time form. The remaining results show that the quantity is invariant under transported kernels and interpretable through covariance operators in the RKHS.

\subsection{Exact and asymptotic risk identities}

The exact risk identity is the kernel analogue of the usual variance formula for a correlated sample mean. Here the sample mean is Hilbert-valued, and the covariance sequence is collapsed to a scalar sequence through the Hilbert inner product.

\begin{theorem}[Exact finite-sample MMD risk]
\label{thm:exact-risk}
Under Assumption~\ref{ass:kernel}, for every $n\ge1$,
\begin{equation}
\label{eq:exact-risk}
R_n(k)=
\frac{1}{n^2}\left\{n\gamma_0+2\sum_{\ell=1}^{n-1}(n-\ell)\gamma_\ell\right\}.
\end{equation}
Consequently, when the denominator is positive,
\begin{equation}
\label{eq:exact-ess-formula}
\ESS_k^{(n)}=
\frac{n^2\gamma_0}{n\gamma_0+2\sum_{\ell=1}^{n-1}(n-\ell)\gamma_\ell},
\end{equation}
with the convention $\ESS_k^{(n)}=+\infty$ if the denominator is zero.
\end{theorem}

For long runs one usually wants a stable asymptotic summary. The next result gives the corresponding integrated-autocorrelation-time representation. It is stated separately because the exact finite-$n$ identity does not require any summability condition, whereas the asymptotic representation does.

\begin{assumption}[Scalar lag summability]
\label{ass:scalar-summability}
The scalar autocovariances satisfy
\[
\sum_{\ell=1}^\infty |\gamma_\ell|<\infty.
\]
\end{assumption}

\begin{theorem}[Asymptotic kernel ESS]
\label{thm:asymp-risk}
Under Assumptions~\ref{ass:kernel} and~\ref{ass:scalar-summability},
\begin{equation}
\label{eq:sigma-k}
nR_n(k)\to \sigma_k^2:=\gamma_0+2\sum_{\ell=1}^\infty\gamma_\ell.
\end{equation}
Moreover $\sigma_k^2\ge0$. If $\sigma_k^2>0$, then
\begin{equation}
\label{eq:asymp-ess}
\frac{\ESS_k^{(n)}}{n}\to \frac{\gamma_0}{\sigma_k^2}
=
\left(1+2\sum_{\ell=1}^\infty\frac{\gamma_\ell}{\gamma_0}\right)^{-1}.
\end{equation}
If $\sigma_k^2=0$, then $R_n(k)=o(n^{-1})$ and $\ESS_k^{(n)}/n\to\infty$ along all $n$ for which $R_n(k)>0$.
\end{theorem}

Thus the kernel integrated autocorrelation time is
\[
\tau_k=\frac{\sigma_k^2}{\gamma_0}=1+2\sum_{\ell=1}^\infty\frac{\gamma_\ell}{\gamma_0},
\]
whenever $\sigma_k^2>0$. The asymptotic kernel ESS is $n/\tau_k$.

\subsection{Invariance under transported kernels}

The preceding definitions become intrinsic only when the kernel is. The cleanest way to state invariance is through transported kernels. This covers changes of chart, isometries, and equivalent representations.

\begin{theorem}[Transported-kernel invariance]
\label{thm:transport}
Let $(M_1,\mathscr B_1)$ and $(M_2,\mathscr B_2)$ be measurable spaces, and let $T:M_1\to M_2$ be measurable. Suppose $k_1$ and $k_2$ satisfy Assumption~\ref{ass:kernel} on $M_1$ and $M_2$ and
\begin{equation}
\label{eq:transport-kernel}
k_2(Tx,Ty)=k_1(x,y),\qquad x,y\in M_1.
\end{equation}
Let $X_t$ be stationary with law $\Pi$ on $M_1$, and set $Y_t=T(X_t)$, whose marginal law is $T_\#\Pi$. Then the finite-sample MMD risks for $(X_t,k_1,\Pi)$ and $(Y_t,k_2,T_\#\Pi)$ are equal for every $n$, and so are the corresponding exact kernel ESS values whenever they are defined.
\end{theorem}

If $M_1=M_2=M$ is a Riemannian manifold and $T$ is an isometry, then any kernel of the form $k(x,y)=\psi(d_g(x,y))$, provided it is positive definite, satisfies \eqref{eq:transport-kernel}. More generally, a chart change preserves the ESS whenever the kernel is specified intrinsically rather than through chart coordinates. This is the sense in which the kernel ESS is coordinate-free.

\subsection{Characteristic kernels and what the ESS measures}

MMD is always a nonnegative discrepancy, but it is not always a metric on probability laws. This distinction matters for interpretation. A non-characteristic kernel can still produce a useful ESS, but it measures accuracy only for the features seen by that kernel.

\begin{proposition}[Metric versus pseudometric interpretation]
\label{prop:characteristic}
Let $\mathcal P$ be a class of probability measures on $M$ for which the kernel mean embedding is defined. If $k$ is characteristic on $\mathcal P$, then $\MMD_k$ is a metric on $\mathcal P$. In that case $\ESS_k^{(n)}$ is the effective independent sample size for estimating $\Pi$ under a genuine distributional metric. If $k$ is not characteristic, $\MMD_k$ is a pseudometric, and $\ESS_k^{(n)}$ remains well-defined but measures effective sample size only for the corresponding kernel features.
\end{proposition}

Thus the paper does not claim that a single scalar number captures every possible notion of Monte Carlo accuracy. The choice of kernel specifies the discrepancy. Reporting kernel ESS over a small family of intrinsic kernels or bandwidths can give a multiscale view of the chain's distributional efficiency.

\subsection{Operator and harmonic-mean representations}

The scalar sequence $\gamma_\ell$ is obtained by taking traces of Hilbert-space covariance operators. The operator view clarifies the relationship between kernel ESS and ordinary scalar ESS. It shows that kernel ESS averages scalar ESS values across the principal directions of the target's feature covariance, with weights determined by the amount of target variation in each direction.

Let $\calH_0$ be the closed linear span of $\{\phi(x):x\in M\}$ in $\Hk$. We regard all operators below as acting on $\calH_0$. This avoids irrelevant nonseparable components of the ambient RKHS. In all compact manifold examples with continuous kernels, $\calH_0$ is separable.

For $u,v\in\calH_0$, write $u\otimes v$ for the rank-one operator
\[
(u\otimes v)(h)=u\inner{v}{h}_{\Hk}.
\]
Define
\[
C=\E\{\phi(X_0)\otimes\phi(X_0)\},\qquad
\Gamma_\ell=\E\{\phi(X_0)\otimes\phi(X_\ell)\}.
\]
Because $\phi$ is bounded, $C$ and each $\Gamma_\ell$ are trace-class. The next result uses a stronger summability assumption than Theorem~\ref{thm:asymp-risk}; it is imposed only for the operator representation.

\begin{assumption}[Trace-norm summability]
\label{ass:trace-sum}
\[
\sum_{\ell\in\mathbb Z}\|\Gamma_\ell\|_1<\infty,
\]
where $\|\cdot\|_1$ denotes the trace norm.
\end{assumption}

\begin{proposition}[Operator trace representation]
\label{prop:operator}
Under Assumptions~\ref{ass:kernel} and~\ref{ass:trace-sum},
\[
S:=\sum_{\ell\in\mathbb Z}\Gamma_\ell
=C+\sum_{\ell=1}^\infty(\Gamma_\ell+\Gamma_\ell^*)
\]
converges absolutely in trace norm. Moreover
\[
\tr(C)=\gamma_0,
\qquad
\tr(S)=\sigma_k^2,
\]
and, when $\sigma_k^2>0$,
\[
\ESS_k^{(n)}\sim n\frac{\tr(C)}{\tr(S)}.
\]
\end{proposition}

Let $C e_j=\lambda_j e_j$ be an orthonormal eigen-decomposition on $\calH_0$, with $\lambda_j\ge0$. Define scalar processes
\[
Y_{j,t}=\inner{\phi(X_t)}{e_j}_{\Hk}.
\]
If $\lambda_j>0$, define the scalar long-run variance
\[
\sigma_j^2=\sum_{\ell\in\mathbb Z}\Cov(Y_{j,0},Y_{j,\ell})
\]
and the corresponding asymptotic scalar ESS for a length-$n$ run by
\[
\ESS_{j,n}=n\lambda_j/\sigma_j^2
\]
when $\sigma_j^2>0$, with $\ESS_{j,n}=+\infty$ when $\sigma_j^2=0$.

\begin{corollary}[Variance-weighted harmonic mean]
\label{cor:harmonic}
Under the assumptions of Proposition~\ref{prop:operator}, if $\sigma_k^2>0$, then
\begin{equation}
\label{eq:harmonic}
\ESS_k^{(n)}
\sim
\frac{\sum_j\lambda_j}{\sum_{j:\lambda_j>0}\lambda_j/\ESS_{j,n}},
\end{equation}
where terms with $\ESS_{j,n}=+\infty$ contribute zero to the denominator.
\end{corollary}

The off-diagonal long-run cross-covariances between different $Y_j$'s do not appear in \eqref{eq:harmonic} because MMD risk is the trace of the long-run covariance operator. The trace keeps only the total variance across an orthonormal basis. Thus the kernel ESS is not an arbitrary scalarization: it is the scalarization dictated by squared RKHS norm loss.

\section{Estimation and precision assessment}
\label{sec:estimation}

The population definition uses $\Pi$ through the centered feature map and uses the infinite lag sum through $\sigma_k^2$. This section gives a practical estimator and a consistency result, the former of which is a kernel version of a spectral variance estimator. It uses only the Gram matrix of the stored sample and therefore does not require explicit coordinates in a vector space.

\subsection{A Gram-matrix lag-window estimator}

Let $K$ be the $n\times n$ Gram matrix with entries $K_{st}=k(X_s,X_t)$. Denote by
\[
H_n=I_n-n^{-1}\mathbf 1_n\mathbf 1_n^T
\]
the centering matrix and define the empirically centered Gram matrix
\[
\widetilde K=H_n K H_n.
\]
Equivalently, if
\[
\bar\Phi_n=\frac1n\sum_{t=1}^n k(X_t,\cdot),
\qquad
\hat\phi_n(x)=k(x,\cdot)-\bar\Phi_n,
\]
then
\[
\widetilde K_{st}=\inner{\hat\phi_n(X_s)}{\hat\phi_n(X_t)}_{\Hk}.
\]
For $0\le \ell<n$, define
\begin{equation}
\label{eq:gamma-hat}
\hat\gamma_\ell=\frac{1}{n-\ell}\sum_{t=1}^{n-\ell}\widetilde K_{t,t+\ell}.
\end{equation}
Given a bandwidth $b_n<n$ and a bounded lag window $w$, define
\begin{equation}
\label{eq:sigma-hat}
\hat\sigma_k^2
=\hat\gamma_0+2\sum_{\ell=1}^{b_n}w\!\left(\frac{\ell}{b_n+1}\right)\hat\gamma_\ell.
\end{equation}
For example, Bartlett weights use $w(u)=(1-u)\mathbf 1(0\le u\le1)$. The estimated kernel ESS is
\begin{equation}
\label{eq:ess-hat}
\widehat\ESS_k
=n\frac{\hat\gamma_0}{\hat\sigma_k^2},
\end{equation}
provided $\hat\sigma_k^2>0$. A nonpositive $\hat\sigma_k^2$ should be reported as an unstable long-run variance estimate or replaced by a positive semidefinite spectral estimator; the population quantity is nonnegative, but finite-sample lag-window estimates need not be.

The estimator above is fully kernelized. It uses the centered Gram matrix rather than coordinates or tangent-space projections. The centering is empirical because $\mu_\Pi$ is unknown, just as ordinary autocovariance estimators replace population means by sample means.

\subsection{Consistency under absolute regularity}

To state a concise consistency theorem, we use an explicit mixing condition. This condition is stronger than strictly necessary, but it keeps the theorem transparent and is satisfied by many geometrically ergodic Markov chains on compact state spaces.

For two $\sigma$-fields $\mathcal A$ and $\mathcal C$, let
\[
\beta(\mathcal A,\mathcal C)
=\frac12\sup\sum_{i,j}|\Prob(A_i\cap C_j)-\Prob(A_i)\Prob(C_j)|,
\]
where the supremum is over finite measurable partitions $\{A_i\}$ of $\mathcal A$ and $\{C_j\}$ of $\mathcal C$. Define the absolute-regularity coefficients of the stationary process by
\[
\beta(m)=\sup_{r\in\mathbb Z}\beta\{\sigma(X_t:t\le r),\sigma(X_t:t\ge r+m)\}.
\]

\begin{assumption}[Dependence and bandwidth]
\label{ass:estimation}
The stationary process is absolutely regular with
\[
\sum_{m=0}^\infty \beta(m)<\infty.
\]
The lag window $w$ is bounded, $w(0)=1$, and $w(u)\to1$ as $u\downarrow0$. The bandwidth satisfies
\[
b_n\to\infty,
\qquad
\frac{b_n^3}{n}\to0.
\]
\end{assumption}

The condition $b_n^3/n\to0$ is convenient rather than sharp. It compensates for the fact that the lag-$\ell$ product sequence involves pairs $(X_t,X_{t+\ell})$, so its variance bound grows at most linearly in $\ell$. This sufficient condition is used only for the consistency proof.  In the finite-sample experiments below we use the common Bartlett-window heuristic \(b=\lfloor n^{1/3}\rfloor\).

\begin{theorem}[Consistency of the kernel ESS estimator]
\label{thm:estimator-consistency}
Assume Assumptions~\ref{ass:kernel} and~\ref{ass:estimation}. Then
\[
\hat\gamma_0\toP \gamma_0,
\qquad
\hat\sigma_k^2\toP \sigma_k^2.
\]
If $\gamma_0>0$ and $\sigma_k^2>0$, then $\Prob(\hat\sigma_k^2>0)\to1$ and
\begin{equation}
\label{eq:ess-consistency}
\frac{\widehat\ESS_k}{n}\toP \frac{\gamma_0}{\sigma_k^2}.
\end{equation}
Equivalently, $\widehat\ESS_k/\ESS_k^{\mathrm{asy}}\toP1$, where $\ESS_k^{\mathrm{asy}}=n\gamma_0/\sigma_k^2$.
\end{theorem}

To sketch, the proof proceeds in three steps. First, absolute regularity and boundedness imply summability of $\gamma_\ell$ and variance bounds for the oracle lag estimators. Second, empirical centering is asymptotically negligible because $\|n^{-1}\sum_t\phi(X_t)\|_{\Hk}=O_p(n^{-1/2})$. Third, truncation and windowing are controlled by absolute summability of the population lags.

\subsection{A kernel-MMD precision rule}

The estimator also gives a direct precision interpretation. Under Theorem~\ref{thm:asymp-risk},
\[
\E[\MMD_k^2(\hat\Pi_n,\Pi)]\approx \frac{\sigma_k^2}{n}.
\]
Thus a mean-square distributional tolerance $\varepsilon^2$ can be encoded by the criterion
\begin{equation}
\label{eq:precision-risk}
\frac{\hat\sigma_k^2}{n}\le \varepsilon^2.
\end{equation}
Equivalently, using \eqref{eq:ess-hat}, this can be written approximately as
\begin{equation}
\label{eq:precision-ess}
\widehat\ESS_k\ge \frac{\hat\gamma_0}{\varepsilon^2}.
\end{equation}
For a family of kernels $k_1,\ldots,k_J$, a multiscale rule can require \eqref{eq:precision-risk} to hold for every $j$, or can report the worst standardized risk $\max_j \hat\sigma_{k_j}^2/(n\varepsilon_j^2)$. This makes explicit the dependence of the stopping rule on the chosen distributional loss and its scale.

\section{Choosing valid kernels on manifolds}
\label{sec:kernels}

The preceding theory applies to any bounded positive-definite kernel. On a manifold, however, constructing such kernels requires care.  The goal is not merely to write a function of a distance, but to ensure positive definiteness.  We review two safe construction principles: Schoenberg-type harmonic expansions on spheres and pullbacks of Euclidean kernels through embeddings or diffeomorphisms.

\subsection{The geodesic Gaussian pitfall}

A tempting choice is
\[
k_h(x,y)=\exp\{-d_g(x,y)^2/h^2\},
\]
for some bandwidth $h > 0$. This is valid in Euclidean space, but it is not generally valid on curved manifolds. Positive definiteness of this geodesic Gaussian for all bandwidths imposes a flatness condition \citep{feragen_2015_GeodesicExponentialKernels}. Thus $k_h$ should not be used on a curved manifold unless positive definiteness has been verified for that manifold and bandwidth. This warning is central: the kernel ESS is mathematically meaningful only if $k$ is positive definite.

\subsection{Schoenberg--Gneiting kernels on spheres}

For spheres, there is a classical and rigorous alternative. Let $\Sphere^{d-1}=\{x\in\R^d:\|x\|=1\}$, $d\ge3$, and let $\lambda=(d-2)/2$. An isotropic kernel has the form
\[
k(x,y)=\psi(x^Ty).
\]
Schoenberg's theorem says that $k$ is positive definite on $\Sphere^{d-1}$ if
\begin{equation}
\label{eq:schoenberg}
\psi(t)=\sum_{m=0}^\infty b_m\frac{C_m^{(\lambda)}(t)}{C_m^{(\lambda)}(1)},
\qquad
b_m\ge0,
\qquad
\sum_{m=0}^\infty b_m<\infty,
\end{equation}
where $C_m^{(\lambda)}$ are Gegenbauer polynomials, and its converse also holds for continuous isotropic positive-definite functions on a fixed sphere \citep{schoenberg_1942_PositiveDefiniteFunctions, gneiting_2013_StrictlyNonstrictlyPositive}. The bound
\[
|C_m^{(\lambda)}(t)|\le C_m^{(\lambda)}(1),\qquad -1\le t\le1,
\]
ensures uniform convergence whenever $\sum_m b_m<\infty$.

A simple family is obtained by taking $b_m=\rho^m$ for $0<\rho<1$:
\begin{equation}
\label{eq:rho-sphere}
k_\rho(x,y)=\sum_{m=0}^\infty \rho^m\frac{C_m^{(\lambda)}(x^Ty)}{C_m^{(\lambda)}(1)}.
\end{equation}
For $\mathbb{S}^2$, where $\lambda=1/2$ and $C_m^{(1/2)}$ are Legendre polynomials $P_m$, this has the closed form
\begin{equation}
\label{eq:s2-poisson}
k_\rho(x,y)=\sum_{m=0}^\infty \rho^m P_m(x^Ty)
=(1-2\rho x^Ty+\rho^2)^{-1/2}.
\end{equation}
This is the kernel used in the sphere experiments below.

\begin{proposition}[Characteristicness on the sphere]
\label{prop:sphere-characteristic}
Let $k$ be an isotropic kernel on $\Sphere^{d-1}$, $d\ge3$, with expansion \eqref{eq:schoenberg}. If $b_m>0$ for every $m\ge0$, then $k$ is characteristic on Borel probability measures on $\Sphere^{d-1}$.
\end{proposition}

For probability measures, strict positivity for $m\ge1$ is already enough because signed differences of probability measures integrate constants to zero. The all-$m$ condition is a simple sufficient condition. In particular, the family \eqref{eq:rho-sphere} is characteristic for every $0<\rho<1$.

\subsection{Pullback kernels from embeddings and diffeomorphisms}

A second safe route is to map the manifold into a Euclidean or Hilbert space where valid kernels are known, and then pull those kernels back. This is not always intrinsic in the sense of being determined by a geodesic distance, but it is mathematically valid and often geometrically natural.

\begin{proposition}[Pullback kernels]
\label{prop:pullback}
Let $\Psi:M\to E$ be a measurable map from $(M,\mathscr B)$ into a measurable subset of a Euclidean space or Hilbert space $E$. If $\widetilde k$ is a positive-definite kernel on $\Psi(M)$, then
\[
k_\Psi(x,y)=\widetilde k\{\Psi(x),\Psi(y)\}
\]
is positive definite on $M$. If, in addition, $\Psi(M)$ is equipped with the trace $\sigma$-field inherited from $E$, $\Psi:M\to\Psi(M)$ is a bimeasurable bijection, and $\widetilde k$ is characteristic on probability measures on $\Psi(M)$, then $k_\Psi$ is characteristic on probability measures on $M$.
\end{proposition}

If a group $G$ acts on $M$ and $\Psi(gx)=U_g\Psi(x)$ for a unitary representation $U_g$ on $E$, then any $\widetilde k$ invariant under $U_g$ gives a $G$-invariant pullback kernel. This formalizes the sense in which some pullback kernels are canonical relative to a chosen geometry.

\begin{example}[Grassmann manifolds]
The Grassmann manifold $\mathrm{Gr}(p,m)$ can be represented by rank-$p$ orthogonal projectors. If $[U]$ denotes the subspace spanned by the columns of an orthonormal matrix $U$, define
\[
\Psi([U])=UU^T.
\]
Then
\[
d_{\mathrm{pr}}([U],[V])=\frac{1}{\sqrt2}\|UU^T-VV^T\|_F
\]
is the projection distance. Therefore
\[
k([U],[V])=\exp\{-\beta d_{\mathrm{pr}}([U],[V])^2\}
\]
is a valid pullback of a Euclidean Gaussian kernel on symmetric matrices. This is a projection-kernel construction, not a geodesic Gaussian in the canonical Grassmann geodesic distance \citep{harandi_2014_ManifoldManifoldGeometryAware, lim_2021_GrassmannianAffineSubspaces}.
\end{example}

\begin{example}[SPD manifolds]
Let $\mathrm{Sym}^+(m)$ denote the cone of symmetric positive-definite matrices. The matrix logarithm is a global diffeomorphism
\[
\log:\mathrm{Sym}^+(m)\to\mathrm{Sym}(m).
\]
The log-Euclidean kernel
\[
k_{\mathrm{LE}}(X,Y)=\exp\{-\beta\|\log X-\log Y\|_F^2\}
\]
is therefore a valid pullback Gaussian kernel and is characteristic, tied to the log-Euclidean geometry \citep{arsigny_2007_GeometricMeansNovel}. It should not be confused with a Gaussian of the affine-invariant geodesic distance.
\end{example}

\begin{example}[Correlation matrices]
The manifold of full-rank correlation matrices is the relative interior of the elliptope in the affine space of symmetric matrices with unit diagonal. Cholesky-derived geometries, including the Euclidean--Cholesky metric and log-Euclidean--Cholesky metric, are pullback geometries through maps that transform correlation matrices into Euclidean coordinates or matrix-log coordinates  \citep{thanwerdas_2022_TheoreticallyComputationallyConvenient}. Let $\Psi_{\mathrm{ECM}}$ or $\Psi_{\mathrm{LECM}}$ denote such a bimeasurable coordinate map onto its image. Then
\[
k(C,D)=\exp\{-\beta\|\Psi(C)-\Psi(D)\|^2\}
\]
is a valid pullback kernel relative to the chosen Cholesky geometry. Different geometries can produce different ESS values. The kernel ESS is coordinate-free after the geometry and kernel have been specified, not independent of that specification.
\end{example}

\section{Experiments}\label{sec:experiments}
\subsection{Rotation dependence of coordinate-wise ESS on the sphere}
\label{subsec:rotation-experiment}

The first experiment isolates the coordinate-dependence issue that motivates the proposed diagnostic.  The goal is not to compare samplers, but to take one fixed Markov chain path on a manifold and examine what happens when the same path is represented in different orthonormal coordinate frames.  If an ESS summary is intrinsic, it should not change under such a representation change.

We considered the unit sphere \(\bbS^2\subset \mathbb R^3\) and generated a Markov chain targeting a unimodal von Mises--Fisher distribution,
\[
    \Pi(dx) = {\rm vMF}(e_3,12)(dx),
    \qquad e_3=(0,0,1)^\top .
\]
The chain was produced by a random-walk Metropolis algorithm with symmetric proposal
\[
    q(y\mid x)={\rm vMF}(x,35)(dy).
\]
After discarding \(1000\) burn-in iterations, we retained \(n=3000\) draws.  The observed acceptance rate was \(0.733\).

Let \(X_1,\ldots,X_n\) denote the retained path.  We then generated \(80\) independent Haar-distributed rotations \(Q_1,\ldots,Q_{80}\in SO(3)\) and formed
\[
    X_t^{(r)} = Q_r X_t,\qquad t=1,\ldots,n.
\]
These rotated sequences represent the same abstract path on $\bbS^2$, viewed under different orthonormal frames.  For each rotated path, we computed ordinary scalar ESS values for the three coordinate time series
\[
    \{(X_t^{(r)})_1\}_{t=1}^n,\qquad
    \{(X_t^{(r)})_2\}_{t=1}^n,\qquad
    \{(X_t^{(r)})_3\}_{t=1}^n.
\]
We also computed the intrinsic kernel ESS using the Schoenberg-valid kernel on \(\bbS^2\)
\[
    k_\rho(x,y)
    =
    \frac{1}{\{1-2\rho x^\top y+\rho^2\}^{1/2}},
    \qquad \rho=0.75.
\]
This is the \(\bbS^2\) closed form of the Legendre expansion
\[
    \sum_{\ell=0}^\infty \rho^\ell P_\ell(x^\top y),
\]
and therefore is positive definite on the sphere.  The long-run variance was estimated using the Bartlett lag-window estimator described in Section~\ref{sec:estimation}, with bandwidth \(b=\lfloor n^{1/3}\rfloor=14\).

The key point is that the kernel Gram matrix is invariant under rotations:
\[
    k_\rho(Qx,Qy)=k_\rho(x,y),\qquad Q\in SO(3).
\]
Thus, the estimated kernel ESS must be unchanged by the rotations up to numerical roundoff.  Coordinate-wise ESS values, on the other hand, need not be preserved because the coordinate projections themselves change from one rotated frame to another.

Table~\ref{tab:exp1-rotation} summarizes the results.  Across all coordinate projections and rotations, the coordinate-wise ESS values had mean \(338.68\), standard deviation \(13.05\), minimum \(326.35\), and maximum \(395.38\).  The range divided by the mean was \(0.204\), even though every value was computed from a rotated representation of the same stored path.  The three coordinate axes behaved similarly, with coefficients of variation between \(0.033\) and \(0.043\).  By contrast, the intrinsic kernel ESS was \(505.24\) for every rotated path; the maximum absolute deviation from the unrotated value was less than \(10^{-6}\) in the reported output.  The corresponding estimated quantities for the unrotated path were
\[
    \widehat\gamma_0 = 1.7782,
    \qquad
    \widehat\sigma_k^2 = 10.5583,
    \qquad
    \widehat{\mathrm{ESS}}_k = 505.24 .
\]

\begin{table}[t]
\centering
\caption{Rotation experiment on \(\bbS^2\).  The same stored path was rotated by \(80\) independent elements of \(SO(3)\).  Coordinate-wise ESS values change with the chosen coordinate frame, while the intrinsic kernel ESS is unchanged up to numerical roundoff.}
\label{tab:exp1-rotation}
\begin{tabular}{lrrrrr}
\hline
Quantity & Mean & SD & Min. & Max. & Range/mean\\
\hline
Coordinate ESS, pooled axes & 338.68 & 13.05 & 326.35 & 395.38 & 0.204\\
Coordinate ESS, first rotated axis & 337.40 & 11.13 & 326.35 & 389.46 & 0.187\\
Coordinate ESS, second rotated axis & 339.72 & 14.45 & 326.43 & 394.27 & 0.200\\
Coordinate ESS, third rotated axis & 338.92 & 13.41 & 326.64 & 395.38 & 0.203\\
Intrinsic kernel ESS & 505.24 & 0.00 & 505.24 & 505.24 & 0.000\\
\hline
\end{tabular}
\end{table}

Figure~\ref{fig:exp1-rotation} gives the same comparison graphically.  The coordinate-wise summaries vary as the coordinate frame is rotated, whereas the kernel ESS is constant.  This experiment illustrates the central distinction between coordinate-based and intrinsic summaries: coordinate ESS values depend on the representation of the path, while the proposed kernel ESS depends only on pairwise spherical relations through \(x^\top y\).

\begin{figure}[t]
\centering
\includegraphics[width=.99\linewidth]{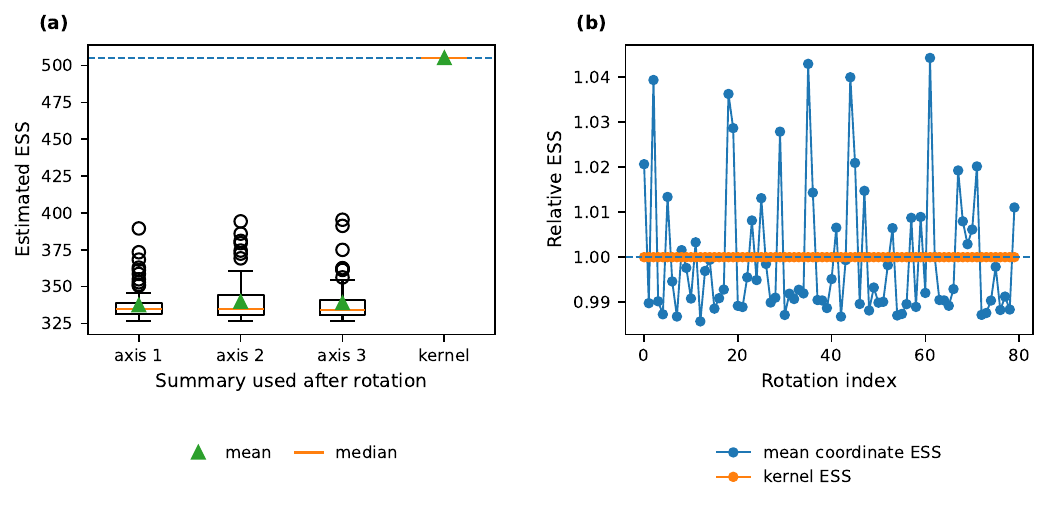}
\caption{
Rotation experiment on \(\bbS^2\).
Panel (a) compares ordinary coordinate-wise ESS values with the intrinsic kernel ESS after applying random rotations to the same stored path.  The coordinate-wise values are shown separately for the three rotated coordinate axes, while the kernel ESS is unchanged; orange lines and green triangles denote medians and means, respectively, and the dashed horizontal line marks the unrotated kernel ESS.  Panel (b) shows relative variation across rotations: the coordinate summary changes with the coordinate frame, whereas the kernel ESS remains equal to its unrotated value up to numerical roundoff.
}
\label{fig:exp1-rotation}
\end{figure}

\subsection{Multimodal spherical mixture and kernel-risk calibration}
\label{subsec:spherical-mixture-experiment}

The second experiment examines whether the proposed kernel ESS is informative for distributional approximation under the kernel loss.  Unlike the previous experiment, which used one fixed path and rotated its representation, this experiment compares two Markov chains with substantially different global mixing behaviour on a multimodal spherical target.  The purpose is twofold: first, to show how the intrinsic ESS varies across kernel scales; and second, to compare the estimated long-run risk constant \(\widehat\sigma_k^2\) with an empirical MMD error computed against a large independent reference sample.

We used a four-component von Mises--Fisher mixture on \(\bbS^2\).  The component means were the vertices of a regular tetrahedron,
\[
\begin{split}
\mu_1&=(1,1,1)^\top/\sqrt 3,\qquad
\mu_2=(1,-1,-1)^\top/\sqrt 3,\\
\mu_3&=(-1,1,-1)^\top/\sqrt 3,\qquad
\mu_4=(-1,-1,1)^\top/\sqrt 3,
\end{split}
\]
with common concentration \(\kappa=28\) and mixture weights
\[
    (w_1,w_2,w_3,w_4)=(0.4,0.3,0.2,0.1).
\]
Thus the target density with respect to spherical surface measure was
\[
    \pi(x)
    =
    \sum_{j=1}^4 w_j c(\kappa)\exp\{\kappa \mu_j^\top x\},
    \qquad x\in \bbS^2,
\]
where \(c(\kappa)\) is the normalizing constant of the von Mises--Fisher density on \(\bbS^2\).  The common concentration ensures that the displayed \(w_j\)'s are the actual mixture weights.

We compared two Metropolis--Hastings chains.  The first was a local random-walk chain with proposal
\[
    q_{\rm loc}(y\mid x)={\rm vMF}(x,90)(dy).
\]
This proposal has high local acceptance but is not designed to cross between well-separated modes.  The second was an independence Metropolis--Hastings chain with a broad proposal mixture
\[
    q_{\rm ind}(y)
    =
    \frac14\sum_{j=1}^4 {\rm vMF}(\mu_j,12)(dy).
\]
This proposal uses the mode locations to provide a deliberately well-mixing reference chain; it is not proposed as a general-purpose sampler.  For each method we generated \(20\) independent replications, discarded \(1000\) burn-in iterations, and retained \(n=2500\) draws.  The iid reference sample used to estimate MMD errors had size \(m=8000\).  The nearest-mode probabilities in the iid reference sample were
\[
    (0.4096,\;0.3003,\;0.1908,\;0.0994),
\]
close to the nominal mixture weights.

For each retained path, we computed the kernel ESS for the Schoenberg-valid \(\bbS^2\) kernel
\[
    k_\rho(x,y)
    =
    \{1-2\rho x^\top y+\rho^2\}^{-1/2},
\]
using \(\rho\in\{0.35,0.60,0.85\}\).  Larger values of \(\rho\) place more weight on higher-order spherical harmonics and therefore probe finer spatial scales.  The long-run variance estimator used the Bartlett lag window with bandwidth \(b=\lfloor n^{1/3}\rfloor=13\).

To compare the empirical MMD error with the risk identity, let \(\hat\Pi_{n,r}\) denote the empirical measure of replication \(r\), and let \(\hat\Pi_m^{\rm ref}\) denote the iid reference empirical measure.  Since the reference measure is also empirical, we used the finite-reference correction
\[
    \widehat D_{\rho,r}
    =
    n\left\{
    \mathrm{MMD}_{k_\rho}^2(\hat\Pi_{n,r},\hat\Pi_m^{\rm ref})
    -
    \frac{\widehat\gamma_{0,\rho}^{\rm ref}}{m}
    \right\},
\]
where \(\widehat\gamma_{0,\rho}^{\rm ref}\) is the iid reference estimate of the one-sample kernel variance.  For stationary output representative of \(\Pi\), the theory predicts
\[
    \widehat D_{\rho,r}\approx \widehat\sigma_{k_\rho,r}^2,
\]
up to Monte Carlo variation and finite-reference error.  We also recorded the total variation error between the chain's nearest-mode frequencies and the iid reference nearest-mode frequencies.

Table~\ref{tab:exp2-mixture} summarizes the kernel-risk quantities.  The independence mixture chain explored the four modes well.  Its estimated intrinsic kernel ESS was about \(900\) across the three kernel scales, corresponding to approximately \(36\%\) of the retained sample size.  Its corrected empirical risk constants \(\widehat D_{\rho,r}\) were close to the estimated long-run variance constants \(\widehat\sigma_{k_\rho}^2\): the average calibration ratios \(\widehat D_{\rho,r}/\widehat\sigma_{k_\rho,r}^2\) were \(1.22\), \(1.25\), and \(1.26\) for \(\rho=0.35,0.60,0.85\), respectively.  Its mean nearest-mode total variation error, shown in Fig.~\ref{fig:exp2-mixture}, was \(0.023\), and its mean acceptance rate was \(0.526\).

The local vMF chain behaved very differently.  Although its mean acceptance rate was high, \(0.734\), it failed to traverse the separated modes.  Its mean nearest-mode total variation error was \(0.709\), and the corrected MMD risk constants were orders of magnitude larger than the corresponding within-chain long-run variance estimates.  For example, at \(\rho=0.60\), the local chain had \(\widehat\sigma_k^2\approx 3.21\), but \(n\widehat{\mathrm{MMD}}_k^2\approx 2407.58\).  This discrepancy is not a contradiction of the theory: the theory concerns stationary representative output, whereas the local chains remain effectively confined to a subset of the target support after burn-in.  The example therefore illustrates an important practical point shared by all ESS diagnostics: ESS estimates dependence in the observed output and should be interpreted together with convergence and distributional checks.

\begin{table}[t]
\centering
\caption{Multimodal spherical mixture experiment.  Values are means over \(20\) replications, with Monte Carlo standard deviations in parentheses.  The column \(n\widehat{\mathrm{MMD}}^2\) reports the finite-reference-corrected quantity \(\widehat D_{\rho,r}\).  The ratio column is \(n\widehat{\mathrm{MMD}}^2/\widehat\sigma_k^2\).  Mode-frequency errors are shown separately in Fig.~\ref{fig:exp2-mixture}.}
\label{tab:exp2-mixture}
\begin{tabular}{lrrrrr}
\hline
Sampler & \(\rho\) & ESS & \(\widehat\sigma_k^2\) &
\(n\widehat{\mathrm{MMD}}^2\) & Ratio\\
\hline
\multirow{3}{*}{\shortstack{Ind.\\ mixture}}
& 0.35 & 908.2 (58.5) & 1.36 (0.08) & 1.63 (1.33) & 1.22 (1.07)\\
& 0.60 & 910.1 (55.0) & 3.68 (0.20) & 4.50 (3.06) & 1.25 (0.91)\\
& 0.85 & 904.6 (45.4) & 14.40 (0.66) & 17.95 (6.67) & 1.26 (0.51)\\
\hline
\multirow{3}{*}{\shortstack{Local\\vMF}}
& 0.35 & 284.6 (8.2) & 0.66 (0.05) & 1072.45 (304.99) & 1642.32 (485.52)\\
& 0.60 & 310.7 (8.9) & 3.21 (0.22) & 2407.58 (683.19) & 754.61 (222.76)\\
& 0.85 & 431.3 (12.6) & 19.15 (0.89) & 4868.58 (1381.62) & 255.49 (75.09)\\
\hline
\end{tabular}
\end{table}

Figure~\ref{fig:exp2-mixture} displays the same comparison graphically.  The independence mixture chain has much larger intrinsic kernel ESS than the local vMF chain at all three kernel scales, while also producing much smaller mode-mass error.  The local chain's high acceptance rate is therefore misleading: it reflects frequent accepted local moves, not adequate exploration of the target distribution.

\begin{figure}[t]
\centering
\includegraphics[width=.99\linewidth]{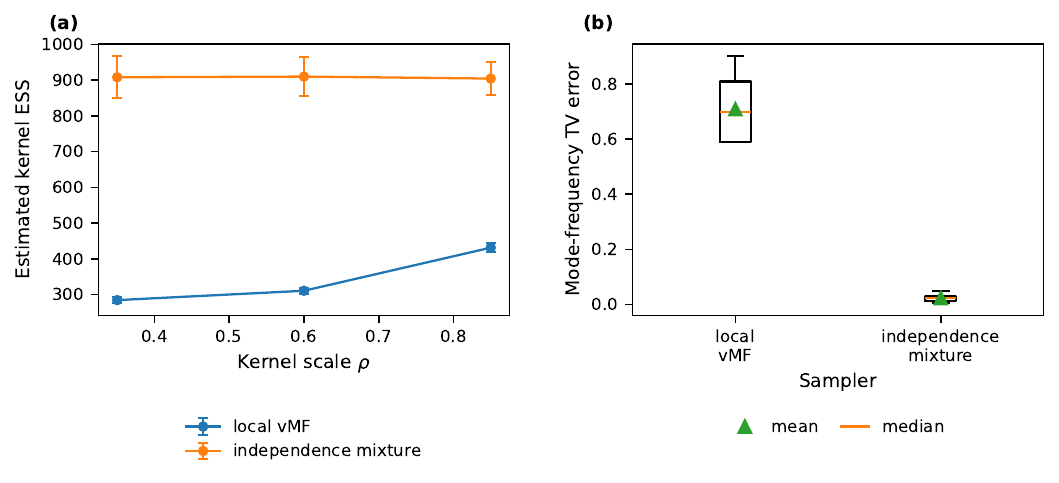}
\caption{
Multimodal spherical mixture experiment.
Panel (a) shows estimated intrinsic kernel ESS over kernel scales \(\rho\in\{0.35,0.60,0.85\}\); points and error bars denote means and one empirical standard deviation over independent replications.  Panel (b) shows total variation error between nearest-mode frequencies from each chain and those from the iid reference sample; orange lines and green triangles denote medians and means, respectively.  The independence mixture chain gives substantially larger intrinsic kernel ESS and smaller mode-mass error than the local vMF random-walk chain.
}
\label{fig:exp2-mixture}
\end{figure}

For completeness, Figure~\ref{fig:exp2-calibration} compares the corrected empirical MMD risk constants with the estimated long-run variance constants.  The independence mixture chain lies near the calibration predicted by the stationary risk identity.  The local vMF chain does not: its empirical MMD error is dominated by failure to visit modes rather than by within-mode autocorrelation.  This reinforces the intended interpretation of the proposed diagnostic.  The intrinsic kernel ESS is a coordinate-free effective sample size for representative MCMC output under a chosen kernel loss; it is not, by itself, a substitute for checking whether the chain has explored the target support.

\begin{figure}[t]
\centering
\includegraphics[width=.99\linewidth]{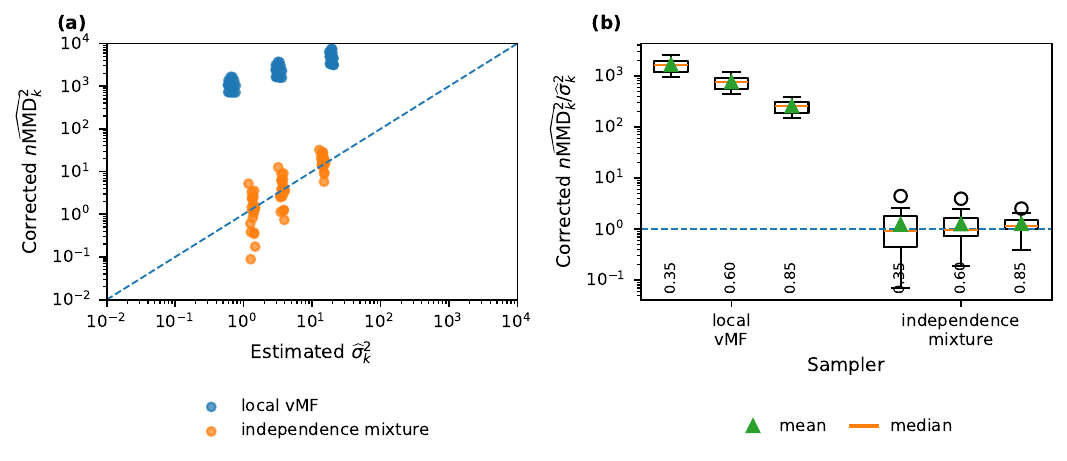}
\caption{
Kernel-risk calibration in the spherical mixture experiment.
Panel (a) compares the corrected empirical distributional error \(n\widehat{\mathrm{MMD}}_k^2\) with the estimated long-run risk constant \(\widehat\sigma_k^2\), with the dashed line denoting equality.  Panel (b) shows the corresponding calibration ratios \(n\widehat{\mathrm{MMD}}_k^2/\widehat\sigma_k^2\), grouped by sampler and kernel scale; orange lines and green triangles denote medians and means, respectively.  The well-mixing independence chain is close to the prediction of the stationary risk identity, whereas the local vMF chain is dominated by non-exploration of modes.
}
\label{fig:exp2-calibration}
\end{figure}

\section{Discussion}
\label{sec:discussion}

Effective sample size is meaningful only after an inferential task has been specified.  For scalar MCMC output the task is estimation of the expectation of a chosen function, and for multivariate output it is estimation of the expectation of a chosen Euclidean vector.  For manifold-valued output there is no canonical vector representation.  The kernel ESS proposed here instead uses distributional approximation as the task: it measures the number of independent samples that would give the same expected squared kernel discrepancy between the empirical law and the target law.  This gives a scalar summary with a finite-sample risk interpretation, an asymptotic autocorrelation representation, and invariance under transformations that preserve the chosen kernel.

The construction is not kernel-free.  The kernel determines which discrepancies between probability laws are emphasized, just as a scalar ESS depends on the selected function and a multivariate ESS depends on the selected vector output.  In applications it may therefore be useful to report kernel ESS over a small collection of geometrically meaningful kernels or bandwidths.  Broad kernels emphasize large-scale distributional features, whereas more localized kernels can be sensitive to finer discrepancies.  This multiscale interpretation is especially natural on manifolds, where local movement and global exploration can differ substantially.

Kernel validity is an important practical issue.  A Gaussian kernel obtained by replacing Euclidean distance with geodesic distance is not generally positive definite on curved manifolds, and therefore does not automatically define an RKHS or an MMD.  The examples in Section~\ref{sec:kernels} illustrate safer constructions: Schoenberg--Gneiting kernels on spheres and pullback kernels arising from projection embeddings or diffeomorphic geometries.  These examples are not exhaustive, but they emphasize that the kernel should be chosen together with the geometry of the state space.

The proposed diagnostic should also not be interpreted as a replacement for convergence assessment.  The population theory concerns representative stationary output and quantifies serial dependence under the chosen kernel loss.  If a chain has not explored the target support, the estimated ESS may describe dependence only within the region visited by the chain.  The spherical mixture experiment illustrates this point: the local vMF chain had high acceptance but poor mode exploration, leading to large distributional error.  In practice, kernel ESS should therefore be reported together with convergence checks, mode-occupancy diagnostics when appropriate, and direct distributional comparisons when reference samples are available.

Several extensions are natural.  The lag-window estimator studied here is one convenient choice; batch-means or other spectral estimators could be developed under alternative dependence conditions.  Data-driven kernel selection is another important direction, since different kernels measure different aspects of Monte Carlo accuracy.  Finally, the same risk-ratio principle could be applied to other intrinsic tasks, such as Fr\'echet mean estimation on manifolds where the mean is unique and regular.  Such extensions would require additional geometric assumptions and are complementary to the distributional kernel approach developed here.

\bibliographystyle{dcu}
\bibliography{references}

@inproceedings{feragen_2015_GeodesicExponentialKernels,
	address = {Boston, MA, USA},
	title = {Geodesic exponential kernels: {When} curvature and linearity conflict},
	isbn = {978-1-4673-6964-0},
	shorttitle = {Geodesic exponential kernels},
	doi = {10.1109/CVPR.2015.7298922},
	urldate = {2022-07-26},
	booktitle = {2015 {IEEE} {Conference} on {Computer} {Vision} and {Pattern} {Recognition} ({CVPR})},
	publisher = {IEEE},
	author = {Feragen, Aasa and Lauze, Francois and Hauberg, Soren},
	month = jun,
	year = {2015},
	pages = {3032--3042},
}

@article{arsigny_2007_GeometricMeansNovel,
	title = {Geometric {Means} in a {Novel} {Vector} {Space} {Structure} on {Symmetric} {Positive}‐{Definite} {Matrices}},
	volume = {29},
	issn = {0895-4798, 1095-7162},
	doi = {10.1137/050637996},
	number = {1},
	urldate = {2021-09-27},
	journal = {SIAM Journal on Matrix Analysis and Applications},
	author = {Arsigny, Vincent and Fillard, Pierre and Pennec, Xavier and Ayache, Nicholas},
	month = jan,
	year = {2007},
	pages = {328--347},
}

@incollection{harandi_2014_ManifoldManifoldGeometryAware,
	address = {Cham},
	title = {From {Manifold} to {Manifold}: {Geometry}-{Aware} {Dimensionality} {Reduction} for {SPD} {Matrices}},
	volume = {8690},
	isbn = {978-3-319-10604-5 978-3-319-10605-2},
	shorttitle = {From {Manifold} to {Manifold}},
	doi = {10.1007/978-3-319-10605-2_2},
	language = {en},
	urldate = {2026-05-04},
	booktitle = {Computer {Vision} – {ECCV} 2014},
	publisher = {Springer International Publishing},
	author = {Harandi, Mehrtash T. and Salzmann, Mathieu and Hartley, Richard},
	editor = {Fleet, David and Pajdla, Tomas and Schiele, Bernt and Tuytelaars, Tinne},
	year = {2014},
	pages = {17--32},
}

@article{thanwerdas_2022_TheoreticallyComputationallyConvenient,
	title = {Theoretically and {Computationally} {Convenient} {Geometries} on {Full}-{Rank} {Correlation} {Matrices}},
	volume = {43},
	issn = {0895-4798, 1095-7162},
	doi = {10.1137/22M1471729},
	language = {en},
	number = {4},
	journal = {SIAM Journal on Matrix Analysis and Applications},
	author = {Thanwerdas, Yann and Pennec, Xavier},
	month = dec,
	year = {2022},
	pages = {1851--1872},
}

@article{sriperumbudur_2011_UniversalityCharacteristicKernels,
	title = {Universality, {Characteristic} {Kernels} and {RKHS} {Embedding} of {Measures}},
	volume = {12},
	number = {70},
	journal = {Journal of Machine Learning Research},
	author = {Sriperumbudur, Bharath K. and Fukumizu, Kenji and Lanckriet, Gert R.G.},
	year = {2011},
	pages = {2389--2410},
}

@article{gneiting_2013_StrictlyNonstrictlyPositive,
	title = {Strictly and non-strictly positive definite functions on spheres},
	volume = {19},
	issn = {1350-7265},
	doi = {10.3150/12-BEJSP06},
	number = {4},
	urldate = {2026-05-02},
	journal = {Bernoulli},
	author = {Gneiting, Tilmann},
	month = sep,
	year = {2013},
}

@article{schoenberg_1942_PositiveDefiniteFunctions,
	title = {Positive definite functions on spheres},
	volume = {9},
	issn = {0012-7094},
	doi = {10.1215/S0012-7094-42-00908-6},
	number = {1},
	urldate = {2026-05-02},
	journal = {Duke Mathematical Journal},
	author = {Schoenberg, I. J.},
	month = mar,
	year = {1942},
}

@article{sejdinovic_2013_EquivalenceDistancebasedRKHSbased,
	title = {Equivalence of distance-based and {RKHS}-based statistics in hypothesis testing},
	volume = {41},
	issn = {0090-5364},
	doi = {10.1214/13-AOS1140},
	number = {5},
	urldate = {2026-05-02},
	journal = {The Annals of Statistics},
	author = {Sejdinovic, Dino and Sriperumbudur, Bharath and Gretton, Arthur and Fukumizu, Kenji},
	month = oct,
	year = {2013},
}

@article{gretton_2012_KernelTwoSampleTest,
	title = {A {Kernel} {Two}-{Sample} {Test}},
	volume = {13},
	issn = {1532-4435},
	number = {null},
	journal = {Journal of Machine Learning Research},
	publisher = {JMLR.org},
	author = {Gretton, Arthur and Borgwardt, Karsten M. and Rasch, Malte J. and Schölkopf, Bernhard and Smola, Alexander},
	month = mar,
	year = {2012},
	pages = {723--773},
}

@article{vats_2019_MultivariateOutputAnalysis,
	title = {Multivariate output analysis for {Markov} chain {Monte} {Carlo}},
	volume = {106},
	issn = {0006-3444, 1464-3510},
	doi = {10.1093/biomet/asz002},
	language = {en},
	number = {2},
	urldate = {2026-04-19},
	journal = {Biometrika},
	author = {Vats, Dootika and Flegal, James M and Jones, Galin L},
	month = jun,
	year = {2019},
	pages = {321--337},
}

@article{lim_2021_GrassmannianAffineSubspaces,
	title = {The {Grassmannian} of affine subspaces},
	volume = {21},
	issn = {1615-3375, 1615-3383},
	doi = {10.1007/s10208-020-09459-8},
	language = {en},
	number = {2},
	urldate = {2022-07-25},
	journal = {Foundations of Computational Mathematics},
	author = {Lim, Lek-Heng and Wong, Ken Sze-Wai and Ye, Ke},
	month = apr,
	year = {2021},
	pages = {537--574},
}

\appendix
\section{Proofs of the population results}
\label{app:proofs}

\begin{proof}[Proof of \eqref{eq:iid-risk}]
Let $\bar\phi_m=m^{-1}\sum_{i=1}^m\phi(Y_i)$, where $Y_i\iid\Pi$. Since $\E\phi(Y_i)=0$ in $\Hk$ and the $Y_i$ are independent,
\[
\E\|\bar\phi_m\|^2
=\frac{1}{m^2}\sum_{i,j=1}^m\E\inner{\phi(Y_i)}{\phi(Y_j)}
=\frac{1}{m^2}\sum_{i=1}^m\E\|\phi(Y_i)\|^2
=\frac{\gamma_0}{m}.
\]
\end{proof}

\begin{proof}[Proof of Theorem~\ref{thm:exact-risk}]
By \eqref{eq:mmd-centred},
\[
R_n(k)=
\E\left\|\frac1n\sum_{t=1}^n\phi(X_t)\right\|^2
=
\frac1{n^2}\sum_{s=1}^n\sum_{t=1}^n
\E\inner{\phi(X_s)}{\phi(X_t)}.
\]
By stationarity this expectation equals $\gamma_{t-s}$, and by symmetry $\gamma_{t-s}=\gamma_{|t-s|}$. There are $n$ pairs with lag zero and, for each $\ell=1,\ldots,n-1$, there are $2(n-\ell)$ ordered pairs with absolute lag $\ell$. This gives \eqref{eq:exact-risk}. Formula \eqref{eq:exact-ess-formula} follows by substituting \eqref{eq:exact-risk} into Definition~\ref{def:kernel-ess}.
\end{proof}

\begin{proof}[Proof of Theorem~\ref{thm:asymp-risk}]
The exact identity gives
\[
nR_n(k)=\gamma_0+2\sum_{\ell=1}^{n-1}\left(1-\frac{\ell}{n}\right)\gamma_\ell.
\]
Since $|1-\ell/n|\le1$ for $\ell<n$ and $\sum_{\ell\ge1}|\gamma_\ell|<\infty$, dominated convergence yields \eqref{eq:sigma-k}. Because $nR_n(k)\ge0$ for every $n$, its limit $\sigma_k^2$ is nonnegative. If $\sigma_k^2>0$, then
\[
\frac{\ESS_k^{(n)}}{n}
=\frac{\gamma_0}{nR_n(k)}\to \frac{\gamma_0}{\sigma_k^2}.
\]
If $\sigma_k^2=0$, then $nR_n(k)\to0$, so $R_n(k)=o(n^{-1})$ and $\ESS_k^{(n)}/n=\gamma_0/(nR_n(k))\to\infty$ whenever $R_n(k)>0$.
\end{proof}

\begin{proof}[Proof of Theorem~\ref{thm:transport}]
Let $\hat\Pi_{n,X}=n^{-1}\sum_t\delta_{X_t}$ and $\hat\Pi_{n,Y}=n^{-1}\sum_t\delta_{Y_t}$. The squared MMD between empirical and target measures can be written in kernel-evaluation form as
\[
\frac1{n^2}\sum_{s,t=1}^n k(X_s,X_t)
-\frac2n\sum_{t=1}^n\int k(X_t,x)\,\Pi(dx)
+\iint k(x,x')\,\Pi(dx)\Pi(dx').
\]
Using $Y_t=T(X_t)$, $T_\#\Pi$ for the target on $M_2$, and the transported-kernel identity \eqref{eq:transport-kernel}, the corresponding expression for $(Y_t,k_2,T_\#\Pi)$ equals the expression for $(X_t,k_1,\Pi)$ pathwise. Taking expectations gives equality of $R_n$. Applying the same identity with $n=1$ gives equality of the corresponding $\gamma_0$ terms, since $R_1(k)=\gamma_0$. Therefore the exact kernel ESS values agree.
\end{proof}

\begin{proof}[Proof of Proposition~\ref{prop:characteristic}]
By definition, a kernel is characteristic on $\mathcal P$ if the mean embedding map $\nu\mapsto\int k(x,\cdot)\,\nu(dx)$ is injective on $\mathcal P$. Thus $\MMD_k(\nu_1,\nu_2)=0$ implies $\nu_1=\nu_2$ exactly when $k$ is characteristic. Otherwise MMD is a pseudometric because it remains nonnegative, symmetric, and satisfies the triangle inequality as a Hilbert-space norm, but may fail to separate distinct probability laws.
\end{proof}

\begin{proof}[Proof of Proposition~\ref{prop:operator}]
Boundedness of $\phi$ implies
\[
\E\|\phi(X_0)\otimes\phi(X_\ell)\|_1
=\E\{\|\phi(X_0)\|\,\|\phi(X_\ell)\|\}<\infty,
\]
so each $\Gamma_\ell$ is a well-defined trace-class Bochner integral. By stationarity,
\[
\Gamma_{-\ell}
=\E\{\phi(X_0)\otimes\phi(X_{-\ell})\}
=\E\{\phi(X_\ell)\otimes\phi(X_0)\}
=\Gamma_\ell^*.
\]
Assumption~\ref{ass:trace-sum} gives absolute convergence of $\sum_{\ell\in\mathbb Z}\Gamma_\ell$ in trace norm and hence the representation of $S$.

The trace is a continuous linear functional on the trace-class operators. Therefore
\[
\tr(\Gamma_\ell)
=\E\,\tr\{\phi(X_0)\otimes\phi(X_\ell)\}
=\E\inner{\phi(X_0)}{\phi(X_\ell)}
=\gamma_\ell.
\]
In particular, $\tr(C)=\tr(\Gamma_0)=\gamma_0$. Taking traces in the absolutely convergent series for $S$ gives
\[
\tr(S)=\sum_{\ell\in\mathbb Z}\tr(\Gamma_\ell)=\sum_{\ell\in\mathbb Z}\gamma_\ell=\sigma_k^2.
\]
The ESS statement follows from Theorem~\ref{thm:asymp-risk}.
\end{proof}

\begin{proof}[Proof of Corollary~\ref{cor:harmonic}]
Since $C$ is positive trace-class, $\tr(C)=\sum_j\lambda_j$. If $\lambda_j=0$, then $\Var(Y_{j,0})=0$, so $Y_{j,t}=0$ almost surely for every $t$ by stationarity; such directions contribute neither to $\tr(C)$ nor to $\tr(S)$. For $\lambda_j>0$,
\[
\sigma_j^2=\sum_{\ell\in\mathbb Z}\Cov(Y_{j,0},Y_{j,\ell}).
\]
Using an orthonormal eigenbasis of $C$ and the trace representation,
\[
\tr(S)=\sum_j\inner{S e_j}{e_j}
=\sum_j\sigma_j^2.
\]
Because $\ESS_{j,n}=n\lambda_j/\sigma_j^2$, we have $\sigma_j^2=n\lambda_j/\ESS_{j,n}$, with the convention that zero long-run variance contributes zero. Hence
\[
\ESS_k^{(n)}\sim n\frac{\sum_j\lambda_j}{\sum_j\sigma_j^2}
=\frac{\sum_j\lambda_j}{\sum_{j:\lambda_j>0}\lambda_j/\ESS_{j,n}}.
\]
\end{proof}

\section{Proofs for valid manifold kernels}
\label{app:kernels}

\begin{proof}[Proof of Proposition~\ref{prop:sphere-characteristic}]
Let $\nu=\nu_1-\nu_2$ be a finite signed measure equal to the difference of two probability measures. By the addition theorem for spherical harmonics, for each degree $m$ there exists a positive constant $a_{m,d}$ and an orthonormal basis $\{Y_{mj}\}_{j=1}^{N(m,d)}$ of degree-$m$ spherical harmonics such that
\[
\frac{C_m^{(\lambda)}(x^Ty)}{C_m^{(\lambda)}(1)}
=a_{m,d}\sum_{j=1}^{N(m,d)}Y_{mj}(x)Y_{mj}(y).
\]
The standard bound $|C_m^{(\lambda)}(t)|\le C_m^{(\lambda)}(1)$ and $\sum_m b_m<\infty$ justify term-by-term integration by dominated convergence. Thus
\[
\iint k(x,y)\,\nu(dx)\nu(dy)
=\sum_{m=0}^\infty b_m a_{m,d}\sum_{j=1}^{N(m,d)}\left(\int Y_{mj}(x)\,\nu(dx)\right)^2.
\]
All terms are nonnegative. If the MMD is zero and $b_m>0$ for every $m$, then all spherical harmonic coefficients of $\nu$ vanish. The linear span of spherical harmonics is dense in $C(\Sphere^{d-1})$, so $\int f\,d\nu=0$ for every continuous $f$. By the Riesz representation theorem, $\nu=0$. Hence $\nu_1=\nu_2$.
\end{proof}

\begin{proof}[Proof of Proposition~\ref{prop:pullback}]
For any $x_1,\ldots,x_n\in M$ and $a_1,\ldots,a_n\in\R$,
\[
\sum_{i,j=1}^n a_i a_j k_\Psi(x_i,x_j)
=\sum_{i,j=1}^n a_i a_j\widetilde k\{\Psi(x_i),\Psi(x_j)\}\ge0,
\]
so $k_\Psi$ is positive definite. For characteristicness, observe that
\[
\MMD_{k_\Psi}(\nu_1,\nu_2)
=
\MMD_{\widetilde k}(\Psi_\#\nu_1,\Psi_\#\nu_2).
\]
If $\widetilde k$ is characteristic, equality to zero implies $\Psi_\#\nu_1=\Psi_\#\nu_2$. Since $\Psi$ is a bimeasurable bijection onto its image, equality of pushforwards implies $\nu_1=\nu_2$.
\end{proof}

\section{Proof of the estimator theorem}
\label{app:estimator-proof}

This appendix proves Theorem~\ref{thm:estimator-consistency}. The constants below may change from line to line and depend only on the kernel bound and the summable sequence of absolute-regularity coefficients.

First define the oracle centered lag estimator
\[
\hat\gamma_\ell^\circ
=\frac{1}{n-\ell}\sum_{t=1}^{n-\ell}\inner{\phi(X_t)}{\phi(X_{t+\ell})},
\qquad 0\le\ell<n.
\]
The observed estimator $\hat\gamma_\ell$ replaces $\phi$ by $\hat\phi_n=\phi-\bar\phi_n$, where
\[
\bar\phi_n=\frac1n\sum_{t=1}^n\phi(X_t).
\]

\begin{lemma}[Population summability]
\label{lem:gamma-summability}
Under Assumptions~\ref{ass:kernel} and~\ref{ass:estimation}, $\sum_{\ell\ge1}|\gamma_\ell|<\infty$.
\end{lemma}

\begin{proof}
The centered kernel
\[
\widetilde k(x,y)=\inner{\phi(x)}{\phi(y)}
\]
is bounded. Moreover $\gamma_\ell=\E\widetilde k(X_0,X_\ell)$ and, if $X'$ is an independent draw from $\Pi$, $\E\widetilde k(X_0,X')=0$. Absolute regularity gives
\[
|\E\widetilde k(X_0,X_\ell)-\E\widetilde k(X_0,X')|
\le C\beta(\ell),
\]
for bounded measurable functions of two variables. Hence $|\gamma_\ell|\le C\beta(\ell)$, and the result follows from $\sum_\ell\beta(\ell)<\infty$.
\end{proof}

\begin{lemma}[Oracle lag estimation]
\label{lem:oracle-lags}
Under Assumptions~\ref{ass:kernel} and~\ref{ass:estimation}, if $b_n^3/n\to0$, then
\[
\sum_{\ell=0}^{b_n}|\hat\gamma_\ell^\circ-\gamma_\ell|\toP0.
\]
\end{lemma}

\begin{proof}
Let $Z_t^{(\ell)}=\inner{\phi(X_t)}{\phi(X_{t+\ell})}$. This sequence is bounded uniformly in $t$ and $\ell$. For $h>\ell$, the variables $Z_0^{(\ell)}$ and $Z_h^{(\ell)}$ are functions of blocks separated by $h-\ell$, so the covariance inequality for absolute regularity gives
\[
|\Cov(Z_0^{(\ell)},Z_h^{(\ell)})|\le C\beta(h-\ell).
\]
For $0\le h\le\ell$, use the trivial bounded covariance bound. Therefore, uniformly for $\ell\le b_n$ and eventually $b_n<n/2$,
\[
\Var(\hat\gamma_\ell^\circ)
\le \frac{C(\ell+1)}{n}.
\]
By Cauchy's inequality,
\[
\E\sum_{\ell=0}^{b_n}|\hat\gamma_\ell^\circ-\gamma_\ell|
\le \sum_{\ell=0}^{b_n}\{\Var(\hat\gamma_\ell^\circ)\}^{1/2}
\le C\frac{b_n^{3/2}}{n^{1/2}}\to0.
\]
Thus the sum converges to zero in $L^1$ and hence in probability.
\end{proof}

\begin{lemma}[Empirical centering]
\label{lem:empirical-centering}
Under Assumptions~\ref{ass:kernel} and~\ref{ass:estimation},
\[
\sum_{\ell=0}^{b_n}|\hat\gamma_\ell-\hat\gamma_\ell^\circ|\toP0.
\]
\end{lemma}

\begin{proof}
By Lemma~\ref{lem:gamma-summability} and Theorem~\ref{thm:exact-risk},
\[
\E\|\bar\phi_n\|^2=R_n(k)=O(n^{-1}).
\]
Thus $\|\bar\phi_n\|=O_p(n^{-1/2})$. Let $B=2K_0^{1/2}$, so $\|\phi(x)\|\le B$. For each lag $\ell$,
\[
\hat\gamma_\ell
=\hat\gamma_\ell^\circ
-\inner{\frac1{n-\ell}\sum_{t=1}^{n-\ell}\phi(X_t)}{\bar\phi_n}
-\inner{\bar\phi_n}{\frac1{n-\ell}\sum_{t=1}^{n-\ell}\phi(X_{t+\ell})}
+\|\bar\phi_n\|^2.
\]
Both empirical averages in the inner products have norm at most $B$. Hence
\[
\sup_{0\le\ell\le b_n}|\hat\gamma_\ell-\hat\gamma_\ell^\circ|
\le 2B\|\bar\phi_n\|+\|\bar\phi_n\|^2.
\]
Multiplying by $b_n+1$ and using $b_n^3/n\to0$, hence $b_n/n^{1/2}\to0$, gives the result.
\end{proof}

\begin{proof}[Proof of Theorem~\ref{thm:estimator-consistency}]
The convergence $\hat\gamma_0\toP\gamma_0$ follows from Lemmas~\ref{lem:oracle-lags} and~\ref{lem:empirical-centering}. For the long-run variance estimator, write
\[
\hat\sigma_k^2-\sigma_k^2
= A_n+B_n+C_n,
\]
where
\[
A_n=(\hat\gamma_0-\gamma_0)+2\sum_{\ell=1}^{b_n}w\!\left(\frac{\ell}{b_n+1}\right)(\hat\gamma_\ell-\gamma_\ell),
\]
\[
B_n=2\sum_{\ell=1}^{b_n}\left\{w\!\left(\frac{\ell}{b_n+1}\right)-1\right\}\gamma_\ell,
\]
and
\[
C_n=-2\sum_{\ell=b_n+1}^{\infty}\gamma_\ell.
\]
The boundedness of $w$ and Lemmas~\ref{lem:oracle-lags}--\ref{lem:empirical-centering} imply $A_n\toP0$. Lemma~\ref{lem:gamma-summability}, boundedness of $w$, and pointwise convergence $w(\ell/(b_n+1))\to1$ for each fixed $\ell$ imply $B_n\to0$ by dominated convergence. The tail term $C_n\to0$ by absolute summability. Therefore $\hat\sigma_k^2\toP\sigma_k^2$. If $\sigma_k^2>0$, then $\Prob(\hat\sigma_k^2>0)\to1$, and the continuous mapping theorem gives
\[
\frac{\widehat\ESS_k}{n}=\frac{\hat\gamma_0}{\hat\sigma_k^2}\toP\frac{\gamma_0}{\sigma_k^2}.
\]
\end{proof}

\end{document}